\documentclass[11pt,a4paper]{article}

\usepackage[margin=1in]{geometry}
\usepackage{amsmath,amssymb,amsthm}
\usepackage{mathtools}
\usepackage{graphicx}
\usepackage{bm}
\usepackage[colorlinks=true,citecolor=blue,linkcolor=blue,urlcolor=blue]{hyperref}
\usepackage{xcolor}

\newtheorem{theorem}{Theorem}
\newtheorem{proposition}{Proposition}
\newtheorem{lemma}{Lemma}

\theoremstyle{definition}

\theoremstyle{remark}
\newtheorem*{remark}{Remark}

\newcommand{\KL}{D_{\mathrm{KL}}}
\newcommand{\E}{\mathbb{E}}
\newcommand{\R}{\mathbb{R}}
\newcommand{\PV}{P}
\newcommand{\QVz}{Q^{0}}
\newcommand{\QVm}{Q^{m}}
\newcommand{\QVbar}{\bar{Q}}
\newcommand{\NK}{N_K}
\newcommand{\NV}{N_V}
\newcommand{\Mfam}{\mathcal{M}}
\newcommand{\tr}{\operatorname{tr}}

\title{Information-Geometric Decomposition of\\
       Generalization Error in Unsupervised Learning}

\author{Gilhan Kim\thanks{Department of Physics and Astronomy,
Seoul National University, Seoul 08826, Republic of Korea, where
this work was performed.}
\thanks{Present address: Department of Statistics and Data Science,
Yonsei University, Seoul 03722, Republic of Korea.\\
Email: \texttt{gilhan.albert.kim@gmail.com}.}}

\date{}

\begin{document}
\maketitle

\begin{abstract}
\sloppy
\emergencystretch=3em
\noindent
We decompose the Kullback--Leibler generalization error
(GE)---the expected KL divergence from the data distribution to
the trained model---of unsupervised learning into three
non-negative components: model error, data bias, and variance.
The decomposition is exact for any e-flat model class and
follows from two identities of information geometry: the
generalized Pythagorean theorem and a dual e-mixture variance
identity. As an analytically tractable demonstration, we apply
the framework to $\epsilon$-PCA, a regularized principal
component analysis in which the empirical covariance is
truncated at rank $\NK$ and discarded directions are pinned at
a fixed noise floor $\epsilon$. Although rank-constrained
$\epsilon$-PCA is not itself e-flat, it admits a technical
reformulation with the same total GE on isotropic Gaussian
data, under which each component of the decomposition takes
closed form. The optimal rank emerges as the cutoff
$\lambda_{\mathrm{cut}}^{*} = \epsilon$---the model retains
exactly those empirical eigenvalues exceeding the noise
floor---with the cutoff reflecting a marginal-rate balance
between model-error gain and data-bias cost. A boundary
comparison further yields a three-regime phase
diagram---retain-all, interior, and collapse---separated by the
lower Marchenko--Pastur edge and an analytically computable
collapse threshold $\epsilon_{*}(\alpha)$, where $\alpha$ is
the dimension-to-sample-size ratio. All claims are verified
numerically.
\end{abstract}

\medskip
\noindent
\textbf{Keywords:}
Generalization error;
Kullback--Leibler divergence;
information geometry;
generalized Pythagorean theorem;
unsupervised learning;
random matrix theory;
Marchenko--Pastur distribution;
principal component analysis;
phase transitions.
\medskip

\section{Introduction}
\label{sec:intro}

Identifying the model complexity that minimizes the generalization error
(GE) is a central problem of statistical learning.
For supervised learning the standard guide is the bias--variance
tradeoff: increasing the model complexity reduces the bias of the
estimator at the cost of inflating its variance, and the optimal
complexity is reached at the balance between these two competing
contributions.
For unsupervised learning, where the goal is to estimate an entire
probability distribution rather than a conditional mean, an analogous
decomposition was missing until recently.

A natural setting in which to ask the question rigorously is
that of \emph{fully visible} generative models---models in
which every random variable is directly observed, with no
latent or hidden variables marginalised out. This class
includes fully visible Boltzmann machines, multivariate
Gaussian models such as $\epsilon$-PCA, naive-Bayes-type
models, and any visible-only Markov random field, and excludes
RBMs, variational autoencoders, and most modern deep generative
models. The strongest form of the decomposition holds when the
model manifold is e-flat (an exponential family in its natural
parameters); when the parametric constraint is non-linear in
those parameters---as for the rank-constrained $\epsilon$-PCA
class---the framework still defines the same three components
but their non-negativity is in general lost
(Proposition~\ref{prop:negative}). For $\epsilon$-PCA itself we
circumvent this obstruction by a technical e-flat
reformulation (Lemma~\ref{lem:rot-equiv}) that has the same
total GE on isotropic data, and apply the decomposition to that
reformulation.

In a previous joint work \cite{Kim2023JSTAT}, the present author and
collaborators proposed that the GE of unsupervised learning likewise
admits a two-component tradeoff,
\begin{equation}
  \mathrm{GE} \;=\; \mathrm{ME} \;+\; \mathrm{DE},
  \label{eq:JSTAT2decomp}
\end{equation}
between a \emph{model error} ME, which decreases as the model becomes
more expressive, and a \emph{data error} DE, which grows because a
finite training set is an imperfect representative of the true
distribution.
The decomposition was corroborated numerically for restricted
Boltzmann machines (RBMs) trained on the two-dimensional Ising model
and on the totally asymmetric simple exclusion process (TASEP), and
the resulting U-shaped GE curves yielded a non-trivial optimal model
size that grows with the complexity of the data.

The decomposition (\ref{eq:JSTAT2decomp}), however, was an empirical
observation. Two questions were left unanswered:
\begin{enumerate}
\item[(Q1)] Is the data error DE itself decomposable into more
            elementary contributions, in particular into a part that
            measures finite-sample bias and a part that measures
            stochasticity of training?
\item[(Q2)] Is there a class of models in which the decomposition can
            be derived from first principles, and in which the optimal
            model complexity can be computed in closed form?
\end{enumerate}

In this paper we answer both questions affirmatively, by placing
the problem inside the framework of information geometry
\cite{AmariNagaoka2000,Amari2016} and combining it with classical
random matrix theory. The information-geometric analysis of
statistical learning has a long history: the generalized
Pythagorean theorem and dual flatness underlie the modern
understanding of KL-divergence minimisation \cite{Csiszar1975}
and of maximum-likelihood estimation in exponential families
\cite{Amari2016}, while singular learning theory and the widely
applicable information criterion (WAIC) \cite{Watanabe2010}
provide a parallel asymptotic decomposition of the Bayes
generalization error in terms of the real log-canonical threshold
of the model. On the random-matrix side, the limiting empirical
spectrum of sample covariance matrices is governed by the
Marchenko--Pastur law
\cite{MarchenkoPastur1967,BaiSilverstein2010,Edelman1988},
and the optimal threshold for hard truncation of singular values
in low-rank denoising is the celebrated $4/\sqrt{3}$ rule of
\cite{GavishDonoho2014}, complemented in the spiked-covariance
setting by the BBP transition of
\cite{BaikBenArousPeche2005,Paul2007,BenaychGeorgesNadakuditi2011}.
Our
decomposition is closer in spirit to the bias--variance picture of
frequentist learning than to the Bayes-loss decomposition of
\cite{Watanabe2010}, while our $\epsilon$-PCA closed-form result is
the natural KL-divergence analogue, in a generative setting, of the
spectral truncation rules of
\cite{GavishDonoho2014,Veraart2016MPPCA}.
Our main contributions are:

\paragraph{(i) Closed-form optimal rank for $\epsilon$-PCA.}
We introduce $\epsilon$-PCA, a regularized principal component
analysis of zero-mean Gaussian data: the empirical covariance
is truncated at rank $\NK$ and the $\NV - \NK$ discarded
directions are pinned at a fixed noise floor $\epsilon > 0$.
For isotropic true covariance, the Marchenko--Pastur spectral
law identifies (Theorem~\ref{thm:optimal-cutoff}) the unique
interior local minimum of GE, given by the remarkably simple
cutoff condition
\begin{equation}
  \lambda_{\mathrm{cut}}^{*} \;=\; \epsilon ;
  \label{eq:optcut-intro}
\end{equation}
that is, the optimal model retains \emph{precisely those empirical
covariance eigenvalues that exceed the intrinsic noise floor
$\epsilon$}. The corresponding optimal rank
$\NK^{*}/\NV = \int_{\epsilon}^{\lambda_{+}(\alpha)}
p_{\mathrm{MP}}(\lambda;\alpha)\,d\lambda$ depends explicitly on
$\epsilon$ and on the aspect ratio $\alpha = \NV/D$.

\paragraph{(ii) Three-regime phase diagram of the optimum.}
Comparing the interior local minimum of (i) with the boundary value
at zero rank, we obtain (Proposition~\ref{prop:phase}) a sharp
three-regime phase diagram for the global optimum: a
\emph{retain-all} phase ($\NK^{*} = \NV$) for $\epsilon$ below the
lower Marchenko--Pastur edge, an \emph{interior} phase in which
$\NK^{*}$ decreases monotonically in $\epsilon$, and a
\emph{collapse} phase ($\NK^{*} = 0$) above an analytically
computable threshold $\epsilon_{*}(\alpha)$. The closed-form
prediction of (i) is verified numerically against direct
brute-force optimization in
Section~\ref{sec:epca-numerics}.

\paragraph{(iii) An information-geometric three-component
decomposition that supplies the proof.}
The closed-form results of (i) and (ii) arise from a general
decomposition of the unsupervised GE that we develop in parallel
and which is of independent interest. Theorem~\ref{thm:3decomp}
states that whenever the model manifold $\Mfam$ is an e-flat
submanifold of the space of probability distributions, the GE
satisfies the exact identity
\begin{align}
  \big\langle \KL(\PV \,\|\, \QVm) \big\rangle_{m}
  \;=\;
  & \underbrace{\KL(\PV \,\|\, \QVz)}_{\text{ME}}
    \;+\;
    \underbrace{\KL(\QVz \,\|\, \QVbar)}_{\text{Data bias}}
  \nonumber \\
  & \;+\;
    \underbrace{\big\langle \KL(\QVbar \,\|\, \QVm) \big\rangle_{m}}_{\text{Variance}},
  \label{eq:3decomp}
\end{align}
with each of the three terms a non-negative KL-type quantity;
$\QVz$ is the m-projection of $\PV$ onto $\Mfam$ and $\QVbar$
is the e-mixture of the trained models $\{\QVm\}_m$. The proof
is by two successive applications of identities from
information geometry: the generalized Pythagorean theorem
\cite[Thm.~3.8]{Amari2016} and a dual e-mixture variance
identity. It formalises the empirical two-component tradeoff of
\cite{Kim2023JSTAT}. The non-negativity of the data-bias term is
conditional on e-flatness; Proposition~\ref{prop:negative}
records that whenever the model manifold fails to be e-flat
(visible marginals of hidden-variable models, rank-constrained
Gaussian models such as $\epsilon$-PCA itself), the data bias
is no longer forced to be non-negative. To apply
Theorem~\ref{thm:3decomp} to $\epsilon$-PCA we therefore
introduce a technical reformulation
(Lemma~\ref{lem:rot-equiv}): replacing the trained
eigen-$\epsilon$-PCA model by a fixed-basis diagonal Gaussian
with the same eigenvalues yields a model on an e-flat
sub-family whose per-realization KL to $\PV$ on isotropic data
is unchanged. The decomposition then gives all three components
in closed form, and minimising their sum yields
Theorem~\ref{thm:optimal-cutoff}; the additive identity is
verified numerically to machine precision in
Fig.~\ref{fig:NKstar-eps}.

\medskip
The remainder of the paper is organized as follows.
Section~\ref{sec:prelim} fixes notation and recalls the elements
of information geometry that we use. Section~\ref{sec:decomp}
states and proves the three-component decomposition
(\ref{eq:3decomp}) and records the e-flatness diagnostic.
Section~\ref{sec:epca} contains the main result: it
develops the $\epsilon$-PCA model, introduces the technical
e-flat reformulation that brings the framework of
Section~\ref{sec:decomp} to bear on it
(Lemma~\ref{lem:rot-equiv}), derives the closed-form optimal
cutoff (\ref{eq:optcut-intro}) (Theorem~\ref{thm:optimal-cutoff})
as a corollary of the resulting decomposition, establishes the
three-regime phase structure of the optimal rank
(Proposition~\ref{prop:phase}), and verifies the closed-form
prediction numerically against brute-force optimization. Section~\ref{sec:it-interpretation}
collects two information-theoretic readings of the
decomposition (cross-entropy and rate-distortion), and
Section~\ref{sec:discussion} concludes.

\section{Preliminaries}
\label{sec:prelim}

\subsection{Notation}
Throughout the paper, $\PV$ denotes the unknown true
distribution on a measurable space $\mathcal{V}$, and
expectations $\langle \cdot \rangle_{m}$ are taken over the
randomness of the training datasets indexed by $m$. The Kullback--Leibler divergence is
$\KL(P\|Q) = \int P\,\log(P/Q)\,dv$ (logarithms are natural).
$H(P) = -\int P\log P\,dv$ denotes the differential or discrete
entropy. We use $\Mfam$ for the model manifold, e-flat sub-families
of an ambient exponential family being our main object of interest;
the natural and expectation parameters are denoted $\theta$ and
$\eta$, respectively, with cumulant function $\psi(\theta)$. For
multivariate-statistics conventions we follow
\cite{Anderson2003}, and for random matrices
\cite{BaiSilverstein2010}:
$\Sigma^{m} = D^{-1} X_m^{\top}X_m$ is the empirical covariance of
$D$ i.i.d.\ samples in $\R^{\NV}$, and the dimension-to-sample-size
ratio is $\alpha = \NV/D$.

\subsection{Generalization error of unsupervised learning}

A learning algorithm produces, from each training dataset, a
trained model $\QVm$ belonging to a parametric family $\Mfam$.
The (Kullback--Leibler) generalization error is
\begin{equation}
  \mathrm{GE} \;:=\; \big\langle \KL(\PV \,\|\, \QVm) \big\rangle_{m}.
  \label{eq:GEdef}
\end{equation}

\subsection{Exponential families and dual flatness}

Following \cite{Amari2016}, an exponential family
\begin{equation}
  Q_{V}(v;\theta)
  \;=\;
  \exp\!\Big[\, \theta^{i} F_{i}(v) \;-\; \psi(\theta) \,\Big]
  \label{eq:expfam}
\end{equation}
forms an \emph{e-flat} submanifold of the space of probability
distributions, with natural parameters $\theta^{i}$ as e-affine
coordinates and the cumulant function $\psi(\theta)$ as the potential.
The expectation parameters
$\eta_{i}(\theta) = \partial_{i}\psi(\theta)$ supply the dual,
m-affine coordinates.

A submanifold $\Mfam$ is \emph{e-flat} if it is described by linear
constraints in $\theta$, and \emph{m-flat} if it is described by linear
constraints in $\eta$.
For any e-flat $\Mfam$, the e-mixture
\begin{align}
  \QVbar
  &\;:=\;
  \exp\!\Big[\,\big\langle \log \QVm \big\rangle_{m} + \bar F\,\Big],
  \nonumber \\
  \bar F
  &\;:=\; -\log\!\int\!\exp\!\big\langle \log \QVm(v)\big\rangle_{m}\,dv ,
  \label{eq:emixture}
\end{align}
of any collection of points $\QVm \in \Mfam$ is itself a point of
$\Mfam$. This closure property is what makes the e-mixture, rather
than the arithmetic mean, the natural notion of average inside an
exponential family.

\subsection{Generalized Pythagorean theorem}

We will repeatedly use the following standard result
\cite[Thm.~3.8]{Amari2016}.

\begin{theorem}[Generalized Pythagorean theorem]
\label{thm:GPT}
Let $\Mfam$ be an e-flat submanifold and let $P$ be an arbitrary
distribution. Let $Q^{*} \in \Mfam$ be the m-projection of $P$ onto
$\Mfam$, i.e.,
\begin{equation*}
  Q^{*} \;=\; \arg\min_{Q \in \Mfam}\, \KL(P \,\|\, Q).
\end{equation*}
Then for every $R \in \Mfam$,
\begin{equation}
  \KL(P \,\|\, R)
  \;=\;
  \KL(P \,\|\, Q^{*})
  \;+\;
  \KL(Q^{*} \,\|\, R) .
  \label{eq:GPT}
\end{equation}
\end{theorem}

The two terms on the right-hand side of (\ref{eq:GPT}) are the squared
``orthogonal'' and ``along-manifold'' components, respectively, of the
divergence from $P$ to $R$.

\section{Three-component decomposition of the generalization error}
\label{sec:decomp}

\subsection{Statement and proof}

Throughout this section, $\Mfam$ denotes an e-flat submanifold of the
space of probability distributions on the visible variables, and we
write
\begin{align}
  \QVz &\;:=\; \arg\min_{Q \in \Mfam}\, \KL(\PV \,\|\, Q),
  \nonumber \\
  \QVbar &\;:=\; \exp\!\Big[\big\langle\log \QVm\big\rangle_{m} + \bar F\Big],
  \label{eq:QVzQVbar}
\end{align}
for the m-projection of $\PV$ onto $\Mfam$ and for the e-mixture of
the trained models, respectively.

\begin{theorem}[Three-component decomposition]
\label{thm:3decomp}
Assume that the trained models $\{\QVm\}_{m}$ all belong to the
e-flat submanifold $\Mfam$. Then
\begin{align}
  \big\langle \KL(\PV \,\|\, \QVm) \big\rangle_{m}
  \;=\;
  & \underbrace{\KL(\PV \,\|\, \QVz)}_{=:\,\mathrm{ME}}
  \;+\;
  \underbrace{\KL(\QVz \,\|\, \QVbar)}_{=:\,\mathrm{Data\ bias}}
  \nonumber \\
  & \;+\;
  \underbrace{\big\langle \KL(\QVbar \,\|\, \QVm) \big\rangle_{m}}_{=:\,\mathrm{Variance}}.
  \label{eq:3decomp-thm}
\end{align}.
\end{theorem}

The three terms in (\ref{eq:3decomp-thm}) capture three
qualitatively distinct sources of error, and it is useful to
visualise them separately.

\begin{itemize}
\item[$\bullet$]
  \textbf{Model error}
  $\mathrm{ME} = \KL(\PV\|\QVz)$ is the \emph{irreducible} error.
  It is the divergence from the true distribution $\PV$ down to
  the closest point $\QVz$ of the model manifold $\Mfam$, and
  would persist even with infinitely many training samples and
  an oracle learner. ME measures how badly $\Mfam$ misspecifies
  $\PV$; it depends only on the geometry of $\Mfam$ and not on
  any particular dataset, and shrinks only by enriching the
  model class.
\item[$\bullet$]
  \textbf{Data bias}
  $\mathrm{Data\ bias} = \KL(\QVz\|\QVbar)$ is the
  \emph{systematic} error introduced by the finite-data training
  process. It measures how far the typical (e-mixture-averaged)
  trained model $\QVbar$ sits from the m-projection $\QVz$ that
  an ideal infinite-data learner would recover. It is a
  property of the algorithm--data interaction averaged over
  dataset randomness, not of any single realisation, and
  vanishes only when the learning procedure is consistent in
  the e-mixture sense ($\QVbar \to \QVz$ as $D\to\infty$).
\item[$\bullet$]
  \textbf{Variance}
  $\mathrm{Variance} = \langle\KL(\QVbar\|\QVm)\rangle_{m}$ is
  the \emph{stochastic} error. It is the average dispersion of
  the trained models $\QVm$ around their centroid $\QVbar$ and
  captures dataset-to-dataset fluctuation of the learner. It is
  independent of how well $\Mfam$ fits $\PV$ and is controlled
  purely by the size and randomness of the training set.
\end{itemize}

Geometrically, ME is an \emph{off-manifold} divergence from
$\PV$ down to $\Mfam$, while data bias and variance are both
\emph{on-manifold} divergences along $\Mfam$---the former
between two distinguished points $\QVz$ and $\QVbar$, the
latter a spread of the random $\QVm$ around $\QVbar$. Each
component responds to a different lever: ME to model capacity,
data bias to algorithmic consistency, and variance to sample
size. A schematic visualization of these three distances is
shown in Figure~\ref{fig:manifolds}.

\begin{proof}
Because $\Mfam$ is e-flat, the e-mixture $\QVbar$ defined in
(\ref{eq:QVzQVbar}) belongs to $\Mfam$. By
Theorem~\ref{thm:GPT} applied with $P = \PV$, $Q^{*} = \QVz$, and
$R = \QVbar$,
\begin{equation}
  \KL(\PV \,\|\, \QVbar)
  \;=\;
  \KL(\PV \,\|\, \QVz)
  \;+\;
  \KL(\QVz \,\|\, \QVbar) .
  \label{eq:GPT-1}
\end{equation}

Next we establish the e-mixture variance identity
\begin{align}
  \big\langle \KL(\PV \,\|\, \QVm) \big\rangle_{m}
  \;=\;\;
  & \KL(\PV \,\|\, \QVbar)
  \nonumber \\
  & \;+\;
  \big\langle \KL(\QVbar \,\|\, \QVm) \big\rangle_{m} ,
  \label{eq:emixture-var}
\end{align}
which holds for any reference distribution $\PV$, requires no
e-flatness hypothesis, and is a direct consequence of the
definition (\ref{eq:emixture}) of $\QVbar$.
The identity (\ref{eq:emixture-var}) is the dual analogue, in
information geometry, of the classical bias--variance decomposition
of the mean squared error around its mean: it expresses the average
divergence from $\PV$ to the random trained models as the
divergence from $\PV$ to their e-mixture centroid plus the
fluctuation of the trained models around that centroid
\cite[\S 2.6 and \S 3.6]{Amari2016}.
Indeed, denoting by $\E_R[\cdot] := \int R(v)\,(\cdot)\,dv$ the
expectation under a distribution $R$, the e-mixture identity
$\langle\log\QVm\rangle_{m} = \log\QVbar - \bar F$ gives, after
taking $\E_{\PV}[\cdot]$ on both sides and rearranging,
\begin{align*}
  \big\langle \KL(\PV \,\|\, \QVm) \big\rangle_{m}
  &= \E_{\PV}\!\left[\log\PV\right]
     - \E_{\PV}\!\left[\langle\log\QVm\rangle_{m}\right]
  \\
  &= \E_{\PV}\!\left[\log\PV\right]
     - \E_{\PV}\!\left[\log\QVbar\right] + \bar F
  \\
  &= \KL(\PV \,\|\, \QVbar) \;+\; \bar F .
\end{align*}
A parallel computation, taking $\E_{\QVbar}[\cdot]$ instead, gives
$\langle \KL(\QVbar \,\|\, \QVm) \rangle_{m} = \bar F$.
(In particular, $\bar F \geq 0$, since the left-hand side is a
non-negative average of KL divergences.) Subtracting these two
identities eliminates $\bar F$ and yields (\ref{eq:emixture-var}).

Substituting (\ref{eq:GPT-1}) into (\ref{eq:emixture-var}) gives
(\ref{eq:3decomp-thm}). Non-negativity of all three terms follows
because each is a Kullback--Leibler divergence.
\qed
\end{proof}

\begin{figure}[!t]
  \centering
  \includegraphics[width=\linewidth]{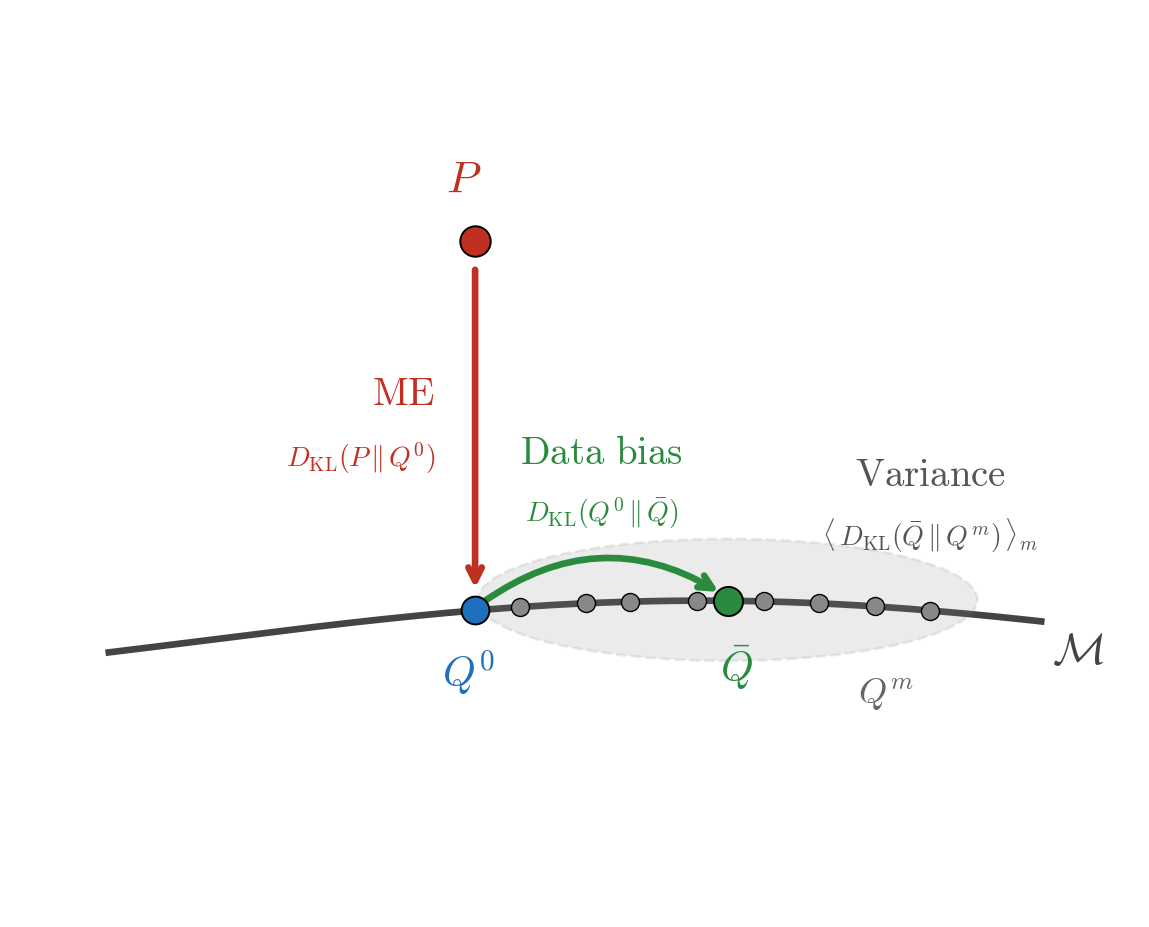}
  \caption{Schematic visualization of the three-component
  decomposition of Theorem~\ref{thm:3decomp} on an e-flat model
  manifold $\Mfam$. The trained models $\{\QVm\}_m$ are scattered
  along $\Mfam$, and their e-mixture $\QVbar$ (green) lies among
  them; the m-projection $\QVz$ (blue) is the closest point of
  $\Mfam$ to the true distribution $\PV$ (red). The three
  components are then three KL-type distances around this
  configuration: the \emph{model error}
  $\mathrm{ME} = \KL(\PV\|\QVz)$ (red arrow) is the off-manifold
  divergence from $\PV$ down to $\Mfam$; the \emph{data bias}
  $\KL(\QVz\|\QVbar)$ (green arrow) is the on-manifold divergence
  from $\QVz$ to the e-mixture centroid; and the
  \emph{variance} $\langle\KL(\QVbar\|\QVm)\rangle_{m}$ (blue
  dashed fan) is the average on-manifold dispersion of the
  trained models around the centroid.}
  \label{fig:manifolds}
\end{figure}

\begin{remark}
Equation (\ref{eq:emixture-var}) makes precise the sense in which the
e-mixture $\QVbar$ is the ``correct'' average for an exponential
family: it is the unique point of the ambient space for which the GE
splits as a sum of a deterministic part $\KL(\PV\|\QVbar)$ and a
fluctuation part $\langle\KL(\QVbar\|\QVm)\rangle_{m}$ that captures
the dataset-to-dataset variability of the trained model. The
e-flatness assumption is needed only for the further decomposition
(\ref{eq:GPT-1}) of the deterministic part.
\end{remark}

\begin{remark}[Two equivalent forms of the data bias]
\label{rem:two-forms}
Two natural expressions for the data bias appear in this work and
in the prior literature:
\begin{align}
  \mathrm{Data\ bias}_{\text{alg}}
  &\;:=\;
  \int \PV(v)\,\log\frac{\QVz(v)}{\QVbar(v)}\,dv
  \nonumber \\
  &\;=\;
  \KL(\PV\|\QVbar) - \KL(\PV\|\QVz),
  \label{eq:databias-alg}
\end{align}
which is a purely algebraic identity in $\PV$, $\QVz$, and
$\QVbar$ and refines the data-error term $\mathrm{DE} :=
\langle\KL(\PV\|\QVm)\rangle_{m} - \KL(\PV\|\QVz)$ of
\cite{Kim2023JSTAT}: by the e-mixture variance identity
(\ref{eq:emixture-var}), one has
$\mathrm{DE} = \mathrm{Data\ bias}_{\text{alg}} +
\mathrm{Variance}$, so the present three-component decomposition
splits the prior two-component data error into a bias and a
variance contribution. The second form of the data bias is
\begin{equation}
  \mathrm{Data\ bias}_{\text{GPT}}
  \;:=\; \KL(\QVz\|\QVbar),
  \label{eq:databias-gpt}
\end{equation}
which appears in the statement of
Theorem~\ref{thm:3decomp}. Under the e-flatness hypothesis of
Theorem~\ref{thm:3decomp}, the generalized Pythagorean identity
(\ref{eq:GPT-1}) gives $\KL(\PV\|\QVbar) = \KL(\PV\|\QVz) +
\KL(\QVz\|\QVbar)$; rearranging and comparing with
(\ref{eq:databias-alg}) yields
$\mathrm{Data\ bias}_{\text{alg}} = \KL(\QVz\|\QVbar) =
\mathrm{Data\ bias}_{\text{GPT}}$, so the two forms coincide
and the decomposition (\ref{eq:3decomp-thm}) holds with either.
The two forms differ when the e-flatness hypothesis fails, however:
the algebraic form (\ref{eq:databias-alg}) continues to satisfy the
additive identity
$\mathrm{GE} = \mathrm{ME} + \mathrm{Data\ bias}_{\text{alg}} +
\mathrm{Variance}$ as a consequence of (\ref{eq:emixture-var}) alone,
while the KL form (\ref{eq:databias-gpt}) does not. The
numerical experiments of Section~\ref{sec:epca-numerics} use
the algebraic form (\ref{eq:databias-alg}) so that the additive
identity holds exactly to machine precision; under the e-flat
hypothesis (which holds for the technical reformulation of
Section~\ref{sec:variant4} but not for eigen-$\epsilon$-PCA
itself) the two forms agree.
\end{remark}

A direct numerical verification of (\ref{eq:3decomp-thm}) on
the e-flat reformulation of $\epsilon$-PCA introduced in
Section~\ref{sec:variant4}, in which all three components are
computed in closed form and compared to the empirical
generalization error to machine precision, is given in
Section~\ref{sec:epca-numerics} and Figure~\ref{fig:NKstar-eps}.

\subsection{Failure of e-flatness and the obstruction}
\label{sec:obstruction}

Of the three components in (\ref{eq:3decomp-thm}), the model error
and the variance are each manifestly non-negative because they are
written from the start as Kullback--Leibler divergences. The data
bias is more delicate: as Remark~\ref{rem:two-forms} discusses,
its algebraic form (\ref{eq:databias-alg}) is a difference of two
KL divergences and is therefore manifestly sign-indefinite, and
its non-negativity in Theorem~\ref{thm:3decomp} (where it equals
$\KL(\QVz\|\QVbar)$) is not built into the definition---it is a
\emph{derived consequence} of the generalized Pythagorean theorem,
holding only when the e-flatness of $\Mfam$ allows the GPT to be
applied. We now show that for models with hidden variables this
hypothesis can fail, and that the algebraic data bias
(\ref{eq:databias-alg}) can indeed take negative values.

Consider an exponential family on the joint space of visible and
hidden variables,
\begin{equation}
  Q_{VH}(v,h;\theta)
  \;=\;
  \exp\!\Big[\,\theta^{i} F_{i}(v,h) - \psi(\theta)\,\Big],
  \label{eq:joint-expfam}
\end{equation}
and its visible marginal, obtained by integrating out (or
summing over) the hidden variables,
\begin{equation}
  Q_{V}(v;\theta) \;=\; \int Q_{VH}(v,h;\theta)\,dh.
  \label{eq:visible-marginal}
\end{equation}
The joint family (\ref{eq:joint-expfam}) is e-flat in the joint
space, but its visible-marginal family
(\ref{eq:visible-marginal}) is in general \emph{not} an
exponential family in $v$ \cite[\S 2.4 and \S 8.1]{Amari2016}:
the marginalization mixes the natural parameters in a
non-affine way.

\begin{proposition}
\label{prop:negative}
Let $\Mfam_{V}^{\mathrm{hid}} := \{ Q_{V}(\cdot;\theta) :
\theta \in \Theta \}$ be the visible-marginal family of an
exponential family with hidden variables, as in
(\ref{eq:visible-marginal}).
\begin{enumerate}
\item[(a)] $\Mfam_{V}^{\mathrm{hid}}$ is generally not e-flat.
\item[(b)] Because the generalized Pythagorean theorem
(Theorem~\ref{thm:GPT}) requires e-flatness, the
m-projection identity used in the proof of
Theorem~\ref{thm:3decomp} need not hold on
$\Mfam_{V}^{\mathrm{hid}}$. The algebraic data bias
(\ref{eq:databias-alg}) is then no longer constrained by any
KL-type inequality and can take a negative value: the
m-projection $\QVz$ onto $\Mfam_{V}^{\mathrm{hid}}$ may be a
strictly worse fit to $\PV$ than the e-mixture $\QVbar$, which
itself generically does not lie on
$\Mfam_{V}^{\mathrm{hid}}$.
\end{enumerate}
\end{proposition}

\begin{proof}[Sketch]
Statement (a) is the well-known fact that marginalization does not,
in general, preserve exponential-family structure
\cite[\S 8]{Amari2016}. An explicit class of examples is provided
by restricted Boltzmann machines \cite{Smolensky1986,Hinton2002}:
even though the joint distribution $Q_{VH}$ on visible and hidden
units is an exponential family in the natural parameters, the
visible marginal $Q_{V}(v) = \sum_{h} Q_{VH}(v,h)$ is in general a
mixture of products of Bernoulli factors (one mixture component per
hidden configuration) and cannot be written in the form
(\ref{eq:expfam}). The exception is the special case in which the
hidden conditional decouples in a particular way; this exceptional
case characterizes the so-called \emph{exponential family
harmoniums} of \cite{WellingRosenZviHinton2005}, and lies at the
boundary of the obstruction discussed here.
Statement (b) then follows immediately from (a) and the
e-flatness hypothesis of Theorem~\ref{thm:GPT}.
\qed
\end{proof}

\begin{remark}[Joint vs.\ visible loss]
\label{rem:joint-vs-visible}
Proposition~\ref{prop:negative} should not be read as the
statement that hidden-variable models are themselves non-e-flat.
The joint family $\{Q_{VH}(\cdot;\theta)\}$ is, by construction
(\ref{eq:joint-expfam}), an exponential family in $(v, h)$ and
hence e-flat in the joint space; Theorem~\ref{thm:3decomp}
applies to it without any modification, provided one defines the
generalization error in the joint space,
$\langle\KL(P_{VH}\,\|\,Q_{VH}^{m})\rangle_{m}$, with respect
to some reference joint distribution $P_{VH}$. The obstruction
arises only when the GE is defined---as is standard in the
unsupervised learning literature---in the \emph{visible} space,
$\langle\KL(\PV\,\|\,Q_{V}^{m})\rangle_{m}$, because the
effective model class is then the visible-marginal family
$\Mfam_{V}^{\mathrm{hid}}$, which is non-e-flat. In short: the
framework's failure for hidden-variable models is a property of
the standard \emph{loss}, not of the model itself. The joint GE
is therefore not a meaningful generalisation target for
ordinary unsupervised training---hidden variables have no
ground-truth values to compare against---but it remains a
well-defined and useful quantity for measuring the divergence
between two trained models whose hidden labellings can be put
in one-to-one correspondence (e.g., a teacher--student or
distillation setting in which the hidden units of the student
inherit a fixed semantic ordering from the teacher).
\end{remark}

\begin{remark}[The sign of the data bias as an e-flatness diagnostic]
The non-negativity of the data bias in the e-flat case is a
derived consequence of the generalized Pythagorean theorem
rather than a feature of its definition, and is therefore lost
whenever the e-flatness hypothesis fails.
Proposition~\ref{prop:negative} thus suggests using the
empirically measured sign of the algebraic data bias
(\ref{eq:databias-alg}) as a diagnostic for whether a
generative model class is well approximated by an exponential
family in its visible variables.
\end{remark}

\begin{remark}[The obstruction is not specific to hidden variables]
\label{rem:not-only-hidden}
Although hidden variables provide the most familiar source of
non-e-flatness, the proposition does not require them. Any
non-affine constraint on the natural parameters of an ambient
exponential family yields a curved (non-e-flat) submanifold and
the same loss of non-negativity for the data bias. A fully
visible example is the rank-constrained $\epsilon$-PCA model
class (\ref{eq:epca-class}) of Section~\ref{sec:epca}: each
$\epsilon$-PCA model is a Gaussian and hence lies in the
(e-flat) Gaussian exponential family on $\R^{\NV}$, but the
rank-plus-noise-floor constraint that characterises the
$\epsilon$-PCA submanifold is non-linear in the natural
parameter $\Sigma^{-1}$. The closed-form analysis of
Section~\ref{sec:epca} therefore proceeds via a technical
e-flat reformulation (Lemma~\ref{lem:rot-equiv}) that allows
Theorem~\ref{thm:3decomp} to be applied to $\epsilon$-PCA
indirectly.
\end{remark}

\section{An analytically tractable model: $\epsilon$-PCA}
\label{sec:epca}

We now turn to a fully visible model in which both the framework
of Theorem~\ref{thm:3decomp} and the underlying optimisation can
be carried through in closed form. The model is a regularised
principal component analysis of zero-mean Gaussian data with a
fixed noise floor $\epsilon$ on the discarded directions; we call
it $\epsilon$-PCA. It is closely related to probabilistic PCA
\cite{TippingBishop1999}, with the role of the residual noise
variance played by the parameter $\epsilon$, but the noise floor
here is held fixed rather than fitted, so as to keep the rank a
one-parameter family controlled by $\NK$.

\subsection{Setup}

The true distribution is the centered isotropic multivariate
normal
\begin{equation}
  \PV(v) \;=\; \frac{1}{\sqrt{(2\pi)^{\NV}}}
              \exp\!\Big[-\tfrac{1}{2}\,v^{\top} v\Big],
  \qquad \Sigma_{0} = I_{\NV}.
  \label{eq:trueGaussian}
\end{equation}
A training dataset of $D$ samples gives an empirical covariance
$\Sigma^{m} := \frac{1}{D}\, X_{m}^{\top}X_{m}$, distributed as a
Wishart matrix $\mathcal{W}_{\NV}(I, D)$.

\paragraph{The $\epsilon$-PCA model class.}
The rank-$\NK$ $\epsilon$-PCA model class is the family of
zero-mean Gaussians whose covariance has at most $\NK$ free
eigendirections and otherwise sits at the noise floor $\epsilon$:
\begin{align}
  \Mfam_{\epsilon\text{-PCA}}^{\NK}
  \;:=\;
  \big\{\,& \mathcal{N}\!\big(0,\, U D U^{\top}\big) :
  \nonumber \\
  & U \in \mathrm{St}(\NK, \NV),\;
  d_{i} > 0,
  \nonumber \\
  & D = \mathrm{diag}(d_{1},\ldots,d_{\NK},
                       \epsilon,\ldots,\epsilon)\big\},
  \label{eq:epca-class}
\end{align}
where $\mathrm{St}(\NK, \NV)$ is the Stiefel manifold of $\NK$
orthonormal vectors in $\R^{\NV}$ specifying the kept directions.
The constraint defining $\Mfam_{\epsilon\text{-PCA}}^{\NK}$ (``at
most $\NK$ eigenvalues exceed $\epsilon$'') is non-linear in the
natural parameter $\theta = -\tfrac{1}{2}\,\Sigma^{-1}$ of the
ambient Gaussian exponential family, so the model manifold is
\emph{not} e-flat (cf.\ Proposition~\ref{prop:negative} and
Remark~\ref{rem:not-only-hidden}). Consequently
Theorem~\ref{thm:3decomp} cannot be applied directly to
$\epsilon$-PCA; in Section~\ref{sec:variant4} we introduce a
technical e-flat reformulation that has the same total
generalization error on isotropic data, and to which the
decomposition does apply.

\paragraph{The trained model.}
For each training dataset $X_m$, we form $\Sigma^{m}$ and extract
its eigendecomposition
$\Sigma^{m} = U^{m}\Lambda^{m}(U^{m})^{\top}$ with
$\lambda_{1}^{(m)} \geq \cdots \geq \lambda_{\NV}^{(m)}$. The
trained $\epsilon$-PCA model is obtained by retaining the top
$\NK$ empirical eigendirections and replacing the remaining
$\NV - \NK$ eigenvalues by $\epsilon$:
\begin{equation}
  \QVm
  \;=\; \mathcal{N}\!\big(0,\, U^{m} D^{m} (U^{m})^{\top}\big),
  \label{eq:QVm-block}
\end{equation}
with
$D^{m} = \mathrm{diag}(\lambda_{1}^{(m)},\ldots,
\lambda_{\NK}^{(m)},\epsilon,\ldots,\epsilon)$. This is the
textbook PCA reconstruction with a noise floor on the discarded
directions.

\paragraph{$\NK$ as a rank, not a hidden dimension.}
Unlike RBMs, $\epsilon$-PCA has no hidden variables: the model
$\QVm$ is a Gaussian on the same $\R^{\NV}$ as the data. The
integer $\NK \leq \NV$ does not count hidden units; it is the
rank parameter of the model, controlling how many free
eigendirections of the empirical covariance are retained.

\medskip
Because $\PV = \mathcal{N}(0, I)$ is rotationally invariant, the
KL divergence $\KL(\PV \,\|\, \QVm)$ depends on $\QVm$ only
through its eigenvalues:
\begin{align}
  2\,\KL(\PV \,\|\, \QVm)
  \;=\;\;
  & \sum_{i=1}^{\NK}\!\Big(\frac{1}{\lambda_{i}^{(m)}}
                            + \log\lambda_{i}^{(m)}\Big)
  \nonumber \\
  & + (\NV - \NK)\!\Big(\frac{1}{\epsilon} + \log\epsilon\Big) - \NV.
  \label{eq:GEepca-realisation}
\end{align}
The eigenvectors $U^{m}$ drop out, and the analysis below
depends only on the empirical eigenvalue spectrum.

\subsection{Marchenko--Pastur asymptotics}

In the high-dimensional limit
$\NV, D \to \infty$ with $\alpha := \NV/D$ fixed, the empirical
spectral density of $\Sigma^{m}$ converges to the
Marchenko--Pastur (MP) law \cite{MarchenkoPastur1967}
\begin{align}
  p_{\mathrm{MP}}(\lambda)
  &\;=\; \frac{1}{2\pi\alpha\lambda}
        \sqrt{(\lambda_{+}-\lambda)(\lambda-\lambda_{-})},
  \nonumber \\
  \lambda_{\pm} &\;=\; (1 \pm \sqrt{\alpha})^{2},
  \label{eq:MPdensity}
\end{align}
supported on $[\lambda_{-},\lambda_{+}]$ (we assume $\alpha < 1$, so
that there is no atom at zero).
We define the cutoff
$\lambda_{\mathrm{cut}} \in [\lambda_{-},\lambda_{+}]$ by
\begin{equation}
  \int_{\lambda_{\mathrm{cut}}}^{\lambda_{+}}
       p_{\mathrm{MP}}(\lambda)\,d\lambda
  \;=\; \frac{\NK}{\NV} \;=:\; r ,
  \label{eq:cutoffdef}
\end{equation}
so that, in the asymptotic limit, retaining the top $\NK$
eigenvalues amounts to retaining all eigenvalues above
$\lambda_{\mathrm{cut}}$.

\subsection{Technical reformulation as an e-flat sub-family}
\label{sec:variant4}

To bring the framework of Section~\ref{sec:decomp} to bear on
$\epsilon$-PCA, we use a technical reformulation of the trained
model that has the same total generalization error on isotropic
data but lives on an e-flat sub-family. Let
\begin{equation}
  \Mfam_{\diamond}^{\NK}
  \;:=\;
  \big\{\,\mathcal{N}(0, \mathrm{diag}(d_{1},\ldots,d_{\NK},
                                        \epsilon, \ldots, \epsilon))
   \,:\, d_{i} > 0\,\big\}
  \label{eq:Mfam-diamond}
\end{equation}
be the coordinate sub-family of the diagonal Gaussian
exponential family on $\R^{\NV}$ obtained by holding the last
$\NV - \NK$ diagonal entries fixed at $\epsilon$. Since
the constraint $d_{\NK+1} = \cdots = d_{\NV} = \epsilon$ is
affine in the natural parameters $\theta_{i} = -1/(2 d_{i})$ of
the diagonal Gaussian family, $\Mfam_{\diamond}^{\NK}$ is e-flat.

Given the empirical eigenvalues
$\lambda_{1}^{(m)} \geq \cdots \geq \lambda_{\NV}^{(m)}$ from the
dataset $X_{m}$, define the trained ``$\diamond$-model''
\begin{equation}
  \widetilde{Q}^{m}
  \;:=\;
  \mathcal{N}\!\big(0,\, \mathrm{diag}(\lambda_{1}^{(m)},\ldots,
                                        \lambda_{\NK}^{(m)},
                                        \epsilon,\ldots,\epsilon)\big)
  \;\in\;\Mfam_{\diamond}^{\NK} ,
  \label{eq:Qtildem}
\end{equation}
i.e., place the top $\NK$ empirical eigenvalues on the diagonal
of $\Mfam_{\diamond}^{\NK}$ in the standard basis. The
$\diamond$-model $\widetilde{Q}^{m}$ is in general not equal
to the eigen-$\epsilon$-PCA model $\QVm$ of (\ref{eq:QVm-block})
as a Gaussian distribution---the latter rotates back to the
empirical eigenbasis via $U^{m}$, the former does not---but on
isotropic data the two have identical KL divergence to $\PV$.

\begin{lemma}[Rotational equivalence on isotropic data]
\label{lem:rot-equiv}
For $\PV = \mathcal{N}(0, I_{\NV})$ and any $\NK \leq \NV$, the
eigen-$\epsilon$-PCA model $\QVm$ of (\ref{eq:QVm-block}) and the
$\diamond$-model $\widetilde{Q}^{m}$ of (\ref{eq:Qtildem})
satisfy
\begin{equation}
  \KL\!\big(\PV \,\big\|\, \QVm\big)
  \;=\;
  \KL\!\big(\PV \,\big\|\, \widetilde{Q}^{m}\big)
  \label{eq:rot-equiv}
\end{equation}
for every realisation $m$. In particular, the two models share
the same total generalization error
$\langle\KL(\PV \,\|\, \cdot)\rangle_{m}$.
\end{lemma}

\begin{proof}
Both models have covariance with eigenvalues
$\{\lambda_{1}^{(m)},\ldots,\lambda_{\NK}^{(m)},\epsilon,
\ldots,\epsilon\}$, differing only in the orthonormal frame in
which those eigenvalues are placed: $U^{m}$ for $\QVm$, and the
identity for $\widetilde{Q}^{m}$. The Gaussian KL
(\ref{eq:gausKL}) below depends on the model covariance only
through $\tr(\Sigma_{1}^{-1}\Sigma_{0})$ and $|\Sigma_{1}|$, both
of which are conjugation-invariant when $\Sigma_{0} = I$. Hence
$\KL(\PV\|\QVm) = \KL(\PV\|\widetilde{Q}^{m})$ realisation by
realisation.
\qed
\end{proof}

The lemma is the bridge that lets us apply the framework of
Section~\ref{sec:decomp}, whose hypothesis (e-flatness) fails for
the eigen-$\epsilon$-PCA model class
(\ref{eq:epca-class})---and for which therefore
Theorem~\ref{thm:3decomp} cannot be invoked
directly---to the $\diamond$-model class
$\Mfam_{\diamond}^{\NK}$, which is e-flat by construction. Any
result we derive about $\langle\KL(\PV\|\widetilde{Q}^{m})\rangle_{m}$
transfers to $\langle\KL(\PV\|\QVm)\rangle_{m}$ via
Lemma~\ref{lem:rot-equiv}. The $\diamond$-model itself does not
appear after Section~\ref{sec:epca-decomp}; it is solely a proof
device.

We pause to motivate this somewhat indirect route.
Theorem~\ref{thm:optimal-cutoff} below could also be obtained by
minimising the total GE expression directly, without ever
introducing the $\diamond$-model or invoking
Theorem~\ref{thm:3decomp}: substitute (\ref{eq:partialsum-MP})
into (\ref{eq:gausKL}), differentiate the resulting scalar
function of $r$, and observe that the Marchenko--Pastur prefactor
cancels in the chain rule. We do not proceed that way for two
reasons. First, the decomposition \emph{exposes the meaning} of
the cutoff condition: it is the balance between the marginal
benefit of removing one $\epsilon$-pinned direction from the
model error and the marginal cost of admitting one
finite-sample fluctuating direction into the data bias. Both
rates have the same functional form $f(\cdot) - 1$, and that
fact---rather than any arithmetic coincidence---is what makes
the cutoff equation solvable in closed form. Second, the
decomposition links Theorem~\ref{thm:optimal-cutoff} to the
general framework of Section~\ref{sec:decomp}: the closed-form
result becomes a direct application of
Theorem~\ref{thm:3decomp} to a particular e-flat sub-family,
rather than an isolated calculation in random matrix theory.

\subsection{Decomposition of GE and the optimal cutoff}
\label{sec:epca-decomp}

For zero-mean Gaussians with covariances $\Sigma_{0}$ and $\Sigma_{1}$,
\begin{equation}
  2\,\KL\!\big(\mathcal{N}(0,\Sigma_{0}) \,\big\|\, \mathcal{N}(0,\Sigma_{1})\big)
  \;=\;
  \tr(\Sigma_{1}^{-1}\Sigma_{0}) - \NV + \log\frac{|\Sigma_{1}|}{|\Sigma_{0}|}.
  \label{eq:gausKL}
\end{equation}
With $\Sigma_{0} = I_{\NV}$ and $\Sigma_{1}$ the diagonal
covariance of $\widetilde{Q}^{m}$, the trace and
log-determinant decompose into contributions from the kept and
from the discarded directions:
\begin{align*}
  \log|\widetilde{Q}^{m}|
  &= \sum_{i=1}^{\NK} \log\lambda_i^{(m)}
   \;+\; (\NV-\NK)\log\epsilon ,
  \\
  \tr\!\big[(\widetilde{Q}^{m})^{-1}\big]
  &= \sum_{i=1}^{\NK} \frac{1}{\lambda_i^{(m)}}
   \;+\; \frac{\NV-\NK}{\epsilon} .
\end{align*}
In the high-dimensional limit, partial sums over the top $\NK$
eigenvalues can be written as integrals against the MP density
restricted to the kept tail, with the asymptotic prefactor $\NV$:
\begin{equation}
  \sum_{i=1}^{\NK} g\!\big(\lambda_i^{(m)}\big)
  \;\xrightarrow{\NV,D\to\infty}\;
  \NV \int_{\lambda_{\mathrm{cut}}}^{\lambda_+}
       g(\lambda)\,p_{\mathrm{MP}}(\lambda)\,d\lambda
  \label{eq:partialsum-MP}
\end{equation}
for any continuous $g$ (consistency with (\ref{eq:cutoffdef}) is
the case $g \equiv 1$). Define
\begin{align}
  I_{-1}(\lambda_{\mathrm{cut}}) &:=
       \int_{\lambda_{\mathrm{cut}}}^{\lambda_{+}}
       \frac{p_{\mathrm{MP}}(\lambda)}{\lambda}\,d\lambda,
  \nonumber \\
  I_{\log}(\lambda_{\mathrm{cut}}) &:=
       \int_{\lambda_{\mathrm{cut}}}^{\lambda_{+}}
       \log\lambda\;p_{\mathrm{MP}}(\lambda)\,d\lambda.
  \label{eq:integralsdef}
\end{align}

\paragraph{Three-component decomposition via Theorem~\ref{thm:3decomp}.}
Because $\Mfam_{\diamond}^{\NK}$ is e-flat,
Theorem~\ref{thm:3decomp} applies to the $\diamond$-models
$\{\widetilde{Q}^{m}\}_{m}$. The m-projection of $\PV$ onto
$\Mfam_{\diamond}^{\NK}$ is found by minimising
$\KL(\PV \,\|\, Q)$ over $Q \in \Mfam_{\diamond}^{\NK}$; with
$\PV = \mathcal{N}(0, I)$ this gives $d_{i} = 1$ for $i \leq \NK$,
\begin{equation}
  \widetilde{Q}^{0}
  \;=\;
  \mathcal{N}\!\big(0,\, \mathrm{diag}(1,\ldots,1,
                                        \epsilon,\ldots,\epsilon)\big).
  \label{eq:Q0diamond}
\end{equation}
The model error $\KL(\PV \,\|\, \widetilde{Q}^{0})$ is
straightforward from (\ref{eq:gausKL}):
\begin{equation}
  \frac{2\,\mathrm{ME}(r)}{\NV}
  \;=\;
  (1-r)\!\left[\,\frac{1}{\epsilon} + \log\epsilon - 1\,\right]
  \;=\;
  (1-r)\,(f(\epsilon) - 1),
  \label{eq:ME-r}
\end{equation}
where $f(x) := 1/x + \log x$ is the per-direction Gaussian KL
``cost function''. ME is linear in $r$ and decreases at the
constant rate $f(\epsilon) - 1 > 0$ per unit increase in the
rank ratio.

The e-mixture $\widetilde{Q}^{\bar{m}}$ of the $\diamond$-models
is, by (\ref{eq:emixture}), the diagonal Gaussian whose first
$\NK$ entries are the harmonic averages of the corresponding
top empirical eigenvalues; in the deterministic MP limit these
averages converge to the deterministic quantile values
$\mu_{i} = \lambda_{(i/\NV)}$ where $\lambda_{(\cdot)}$ is the
inverse of the upper-tail MP CDF. A direct computation using
(\ref{eq:partialsum-MP}) gives
\begin{equation}
  \frac{2\,\mathrm{Data\ bias}(r)}{\NV}
  \;=\;
  I_{-1}(\lambda_{\mathrm{cut}}) + I_{\log}(\lambda_{\mathrm{cut}}) - r .
  \label{eq:DB-r}
\end{equation}
The variance $\langle\KL(\widetilde{Q}^{\bar{m}}\,\|\,
\widetilde{Q}^{m})\rangle_{m}$ vanishes in the deterministic
MP limit (the e-mixture coincides with the typical realisation),
so that asymptotically
\begin{equation}
  \mathrm{GE}(r) \;=\; \mathrm{ME}(r) + \mathrm{Data\ bias}(r) ,
  \label{eq:GE-decomp-asymp}
\end{equation}
with both terms manifestly non-negative by
Theorem~\ref{thm:3decomp}.

\paragraph{First-order condition.}
Differentiating (\ref{eq:ME-r}) and (\ref{eq:DB-r}) and using
$d\lambda_{\mathrm{cut}}/dr = -1/p_{\mathrm{MP}}(\lambda_{\mathrm{cut}})$
together with
$dI_{-1}/d\lambda_{\mathrm{cut}} = -p_{\mathrm{MP}}/\lambda_{\mathrm{cut}}$,
$dI_{\log}/d\lambda_{\mathrm{cut}} = -p_{\mathrm{MP}}\log\lambda_{\mathrm{cut}}$,
the prefactor $p_{\mathrm{MP}}$ cancels and one finds
\begin{align}
  \frac{d}{dr}\!\left[\frac{2\,\mathrm{ME}}{\NV}\right]
  &\;=\;
  -\,(f(\epsilon) - 1),
  \nonumber \\
  \frac{d}{dr}\!\left[\frac{2\,\mathrm{Data\ bias}}{\NV}\right]
  &\;=\;
  f(\lambda_{\mathrm{cut}}) - 1 .
  \label{eq:dME-dDB}
\end{align}
Adding the two and using
(\ref{eq:GE-decomp-asymp}) gives
\begin{equation}
  \frac{d}{dr}\!\left[\frac{2\,\mathrm{GE}}{\NV}\right]
  \;=\;
  f\!\big(\lambda_{\mathrm{cut}}\big) \;-\; f(\epsilon).
  \label{eq:dGEdr}
\end{equation}
The optimal rank therefore sits at the $r$ for which the
marginal cost of one extra free direction in the data bias
($\sim f(\lambda_{\mathrm{cut}}) - 1$) exactly cancels the
marginal benefit of removing one $\epsilon$-pinned direction from ME
($\sim f(\epsilon) - 1$). The two rates have the same form
$f(\cdot) - 1$, which is what makes the cutoff condition
$f(\lambda_{\mathrm{cut}}^{*}) = f(\epsilon)$ closed-form.

The first-order condition $f(\lambda_{\mathrm{cut}}) = f(\epsilon)$
admits, in general, two roots, because $f$ is strictly convex with
global minimum $f(1) = 1$, strictly decreasing on $(0,1)$ and
strictly increasing on $(1,\infty)$. For $\epsilon \in (0,1)$ the
two roots are $\lambda_{\mathrm{cut}} = \epsilon$ and
$\lambda_{\mathrm{cut}} = \widetilde{\lambda} > 1$.

To select the relevant root, follow the trajectory of
(\ref{eq:dGEdr}) as $r$ ranges over $[0,1]$. At $r = 0$,
$\lambda_{\mathrm{cut}} = \lambda_{+}$; as $r$ grows
$\lambda_{\mathrm{cut}}$ first reaches $\widetilde{\lambda}$ (if
this value lies in the MP support), then $1$ (at which
(\ref{eq:dGEdr}) equals $1 - f(\epsilon) < 0$), then $\epsilon$,
and finally $\lambda_{-}$. The derivative therefore changes sign
$+\!\to\!-$ at $\lambda_{\mathrm{cut}} = \widetilde{\lambda}$
(a local \emph{maximum} of GE) and $-\!\to\!+$ at
$\lambda_{\mathrm{cut}} = \epsilon$ (a local \emph{minimum}).

\begin{theorem}[Optimal cutoff in $\epsilon$-PCA]
\label{thm:optimal-cutoff}
In the high-dimensional limit, for
$\epsilon \in (\lambda_{-}(\alpha), 1)$, the unique interior
local minimum of the generalization error
$\langle\KL(\PV \,\|\, \QVm)\rangle_{m}$ of
eigen-$\epsilon$-PCA on isotropic Gaussian data with intrinsic
noise floor $\epsilon$ is attained at the cutoff
\begin{equation}
  \boxed{\;\lambda_{\mathrm{cut}}^{*} \;=\; \epsilon\;}
  \label{eq:optcut}
\end{equation}
and the corresponding optimal rank is
\begin{equation}
  \NK^{*} \;=\; \NV \int_{\epsilon}^{\lambda_{+}(\alpha)}
                  p_{\mathrm{MP}}(\lambda;\alpha)\,d\lambda .
  \label{eq:NKopt}
\end{equation}
Equivalently, the optimal model retains exactly those empirical
covariance eigenvalues that exceed the intrinsic noise floor
$\epsilon$. The other root $\widetilde{\lambda} > 1$ of the
first-order condition $f(\lambda) = f(\epsilon)$ is a local
\emph{maximum} of GE on the kept branch.
\end{theorem}

\begin{proof}
The decomposition (\ref{eq:GE-decomp-asymp}) of the
$\diamond$-model GE into ME~+~Data bias is a direct application
of Theorem~\ref{thm:3decomp} to the e-flat sub-family
$\Mfam_{\diamond}^{\NK}$. The first-order condition
(\ref{eq:dGEdr}) and the sign analysis preceding the theorem
identify $\lambda_{\mathrm{cut}}^{*} = \epsilon$ as the unique
interior local minimum of GE on $r \in (0,1)$. To verify
non-degeneracy, differentiate (\ref{eq:dGEdr}) once more in $r$:
\begin{align*}
  \frac{d^{2}}{dr^{2}}\!\left[\frac{2\,\mathrm{GE}}{\NV}\right]
  &\;=\; f'(\lambda_{\mathrm{cut}})\,
         \frac{d\lambda_{\mathrm{cut}}}{dr}
  \;=\; -\,\frac{f'(\lambda_{\mathrm{cut}})}{p_{\mathrm{MP}}(\lambda_{\mathrm{cut}})}.
\end{align*}
Evaluated at the critical point $\lambda_{\mathrm{cut}} = \epsilon$
and using $f'(\epsilon) = (\epsilon - 1)/\epsilon^{2}$,
\begin{equation*}
  \frac{d^{2}}{dr^{2}}\!\left[\frac{2\,\mathrm{GE}}{\NV}\right]_{\lambda_{\mathrm{cut}}=\epsilon}
  \;=\;
  -\,\frac{(\epsilon - 1)/\epsilon^{2}}{p_{\mathrm{MP}}(\epsilon)}
  \;=\;
  \frac{(1 - \epsilon)}{\epsilon^{2}\,p_{\mathrm{MP}}(\epsilon)}
  \;>\; 0,
\end{equation*}
since $\epsilon < 1$ and $p_{\mathrm{MP}}(\epsilon) > 0$ inside
the MP support, confirming that $\lambda_{\mathrm{cut}}^{*} =
\epsilon$ is a strict local minimum. Substituting
$\lambda_{\mathrm{cut}}^{*} = \epsilon$ into (\ref{eq:cutoffdef})
yields (\ref{eq:NKopt}). The conclusion is stated for the
$\diamond$-model
$\langle\KL(\PV \,\|\, \widetilde{Q}^{m})\rangle_{m}$;
Lemma~\ref{lem:rot-equiv} transfers it to the eigen-$\epsilon$-PCA
GE $\langle\KL(\PV \,\|\, \QVm)\rangle_{m}$.
\qed
\end{proof}

\begin{remark}[Why the cutoff condition is not as trivial as it looks]
\label{rem:not-trivial}
The conclusion ``retain those eigenvalues exceeding the noise
floor $\epsilon$'' coincides with the standard denoising
heuristic, but the closed-form derivation hides three points
that are not at all evident a priori.

\emph{(a) Cancellation of the Marchenko--Pastur prefactor.}
The first-order condition $f(\lambda_{\mathrm{cut}}) = f(\epsilon)$
emerges only because the chain-rule expansion of
$d\mathrm{GE}/dr$ produces the same prefactor
$p_{\mathrm{MP}}(\lambda_{\mathrm{cut}})$ from
$d\lambda_{\mathrm{cut}}/dr$ as from $dI_{-1}/d\lambda_{\mathrm{cut}}$
and $dI_{\log}/d\lambda_{\mathrm{cut}}$, and the prefactor
\emph{cancels}. Were the loss anything other than KL---for
instance, mean-squared error---this cancellation would not occur
and the optimal cutoff would depend explicitly on the local
density $p_{\mathrm{MP}}(\lambda_{\mathrm{cut}})$ and hence on
the aspect ratio $\alpha$. The marginal rates
(\ref{eq:dME-dDB}) make the cancellation transparent: the
$r$-derivatives of ME and Data~bias are both of the form
$f(\cdot) - 1$, and their balance is therefore a transcendental
equation in $f$ alone, independent of the spectral measure.

\emph{(b) Independence from $\alpha$.}
A direct consequence of (a) is that
$\lambda_{\mathrm{cut}}^{*}$ does not depend on the
dimension-to-sample-size ratio $\alpha = \NV/D$: it is the same
function of $\epsilon$ at $\alpha = 0.1$ as at $\alpha = 0.9$.
This contrasts sharply with the optimal hard-threshold rule of
\cite{GavishDonoho2014}, in which the threshold scales as
$(4/\sqrt{3})\sqrt{D \cdot \mathrm{noise}}$ with explicit
$\alpha$-dependence through the upper Marchenko--Pastur edge,
and reflects the fact that KL is a unitless information measure
while Frobenius loss carries dimensional weight.

\emph{(c) The other root is a local maximum, not a minimum.}
The equation $f(\lambda) = f(\epsilon)$ has two roots in
$(0,\infty)$, namely $\epsilon$ itself and a strictly larger
value $\widetilde{\lambda} > 1$. The second root corresponds to
a local \emph{maximum} of the GE on the kept branch and would
be invisible to a per-direction unconstrained analysis: the
unconstrained per-direction rule
$f(\lambda_i) < f(\epsilon)$ would in fact require the largest
empirical eigenvalues (those with $\lambda_i > \widetilde{\lambda}$,
which can lie inside the MP bulk for moderate $\alpha$) to be
\emph{discarded}, since fitting them with their large empirical
value makes the model overconfident and increases the KL more
than the noise-floor approximation does. The rank constraint of
$\epsilon$-PCA forces these large eigenvalues to be retained as
soon as $\NK$ exceeds their count, producing an
\emph{overshoot} cost. Theorem~\ref{thm:optimal-cutoff}
nevertheless asserts that this overshoot does not move the
optimal cutoff away from $\epsilon$ on the lower side---a fact
that is geometric, not arithmetic, and follows from the
linearity of $\mathrm{ME}(r)$ in $r$ once $\Mfam_{\diamond}^{\NK}$
is recognised as e-flat.
\end{remark}

\begin{remark}
Theorem~\ref{thm:optimal-cutoff} produces the strikingly simple
prescription that an $\epsilon$-PCA model should retain precisely
those data directions whose empirical variance exceeds the model's
intrinsic noise floor $\epsilon$. The condition is the natural
analogue, in the $\epsilon$-PCA setting, of the hard-threshold
singular-value rule of \cite{GavishDonoho2014} and of the
noise-only PCA criterion based on the Marchenko--Pastur edge
\cite{Veraart2016MPPCA}; here, however, the threshold is determined
by a model parameter ($\epsilon$) rather than by the upper edge of
the MP density.
\end{remark}

\subsection{Three-regime phase structure of the global optimum}
\label{sec:phase}

Theorem~\ref{thm:optimal-cutoff} identifies the unique interior
local minimum of GE on the open rank ratio $r \in (0,1)$. Whether
this interior minimum is in fact the \emph{global} optimum
depends on a comparison with the two boundary values
$\mathrm{GE}(r{=}0)$ and $\mathrm{GE}(r{=}1)$, and on whether the
interior cutoff $\lambda_{\mathrm{cut}}^{*} = \epsilon$ lies
inside the MP support $[\lambda_{-}, \lambda_{+}]$. Carrying out
this comparison yields the following sharp three-regime phase
structure.

\begin{proposition}[Three-regime phase structure]
\label{prop:phase}
Fix the aspect ratio $\alpha \in (0, 1)$ and consider
$\epsilon$-PCA on isotropic Gaussian data in the
high-dimensional limit $\NV, D \to \infty$ with
$\NV/D = \alpha$. Then the rank ratio
$r^{*} = \NK^{*}/\NV$ that globally minimizes the
generalization error (\ref{eq:GE-decomp-asymp}) satisfies:
\begin{enumerate}
\item[(R1)] (\textbf{retain-all}) \;
  If $\epsilon \leq \lambda_{-}(\alpha) = (1-\sqrt{\alpha})^{2}$,
  then $r^{*} = 1$. Equivalently, the cutoff
  $\lambda_{\mathrm{cut}}^{*}$ lies below the MP support, every
  empirical eigenvalue exceeds $\epsilon$, and the optimal model
  retains all of them.
\item[(R2)] (\textbf{interior}) \;
  If $\lambda_{-}(\alpha) < \epsilon < \epsilon_{*}(\alpha)$,
  then $r^{*}$ is the interior optimum of
  Theorem~\ref{thm:optimal-cutoff},
  $r^{*} = \int_{\epsilon}^{\lambda_{+}(\alpha)}
   p_{\mathrm{MP}}(\lambda;\alpha)\,d\lambda$, and the optimal
  model retains exactly those empirical eigenvalues that exceed
  $\epsilon$.
\item[(R3)] (\textbf{collapse}) \;
  If $\epsilon \geq \epsilon_{*}(\alpha)$, then $r^{*} = 0$. The
  optimal model is the pure noise-floor distribution
  $\mathcal{N}(0,\,\epsilon I_{\NV})$ and uses no information
  from the training data.
\end{enumerate}
The collapse threshold $\epsilon_{*}(\alpha) \in
(\lambda_{-}(\alpha), 1)$ is the unique solution of
\begin{equation}
  \mathrm{GE}_{\mathrm{interior}}(\epsilon)
  \;=\;
  \tfrac{1}{2}\,\NV\!\Big(\tfrac{1}{\epsilon} + \log\epsilon - 1\Big),
  \label{eq:eps-star}
\end{equation}
where the right-hand side is the closed-form value of
$\mathrm{GE}(r{=}0)$ and the left-hand side is the value of
(\ref{eq:GE-decomp-asymp}) at the interior cutoff
$\lambda_{\mathrm{cut}}^{*} = \epsilon$.
\end{proposition}

\begin{proof}
On the half-line $r = 0$ the trained model degenerates to
$\mathcal{N}(0,\, \epsilon I_{\NV})$ and (\ref{eq:gausKL}) gives
$\mathrm{GE}(0) = \tfrac{1}{2}\NV(1/\epsilon + \log\epsilon - 1)$,
the right-hand side of (\ref{eq:eps-star}). On the half-line
$r = 1$ the model is the eigen-MLE and one verifies from
(\ref{eq:dGEdr}) that
$\mathrm{GE}(1) - \mathrm{GE}(\epsilon)$ has the sign of
$\epsilon - \lambda_{-}(\alpha)$ (the MP integrand has support
on $[\lambda_{-},\lambda_{+}]$).

(R1) Suppose $\epsilon \leq \lambda_{-}(\alpha)$. We show that
$f(\lambda) \leq f(\epsilon)$ for every $\lambda \in
[\lambda_{-}, \lambda_{+}]$, so that the right-hand side of
(\ref{eq:dGEdr}) is non-positive on $r \in (0,1)$ and the GE is
non-increasing in $r$, attaining its minimum at $r^{*} = 1$.

Since $f$ is strictly convex on $(0, \infty)$ with global
minimum $f(1) = 1$, the maximum of $f$ on the closed interval
$[\lambda_{-}, \lambda_{+}]$ is attained at one of the
endpoints. For the lower endpoint we have $\epsilon \leq
\lambda_{-} < 1$, and $f$ is strictly decreasing on $(0, 1)$,
so $f(\lambda_{-}) \leq f(\epsilon)$. For the upper endpoint
we use the explicit form $\lambda_{\pm} = (1 \pm \sqrt{\alpha})^{2}$:
writing $x := \sqrt{\alpha} \in (0, 1)$, a direct computation
gives
\begin{equation*}
  f(\lambda_{-}) - f(\lambda_{+})
  \;=\;
  \frac{4x}{(1-x^{2})^{2}}
  \,-\, 2\log\!\frac{1+x}{1-x},
\end{equation*}
which vanishes at $x = 0$ and whose derivative in $x$ is
\begin{equation*}
  \frac{d}{dx}\bigl[f(\lambda_{-}) - f(\lambda_{+})\bigr]
  \;=\;
  \frac{4x^{2}(5 - x^{2})}{(1 - x^{2})^{3}},
\end{equation*}
strictly positive on $(0, 1)$ since the numerator
$4x^{2}(5 - x^{2})$ and the denominator $(1 - x^{2})^{3}$ are
both positive on that interval. Hence
$f(\lambda_{-}) > f(\lambda_{+})$ for every
$\alpha \in (0, 1)$. Combining, $f(\lambda_{+}) <
f(\lambda_{-}) \leq f(\epsilon)$, so $f \leq f(\epsilon)$
throughout the MP support, as claimed.

(R2)--(R3) Suppose $\epsilon > \lambda_{-}(\alpha)$. Then the
interior critical point $\lambda_{\mathrm{cut}}^{*} = \epsilon$
of Theorem~\ref{thm:optimal-cutoff} lies in the MP support and
is a strict local minimum on the open interval. By continuity
of (\ref{eq:GE-decomp-asymp}) on $[0,1]$, the global minimizer
is one of $\{0,\, r_{\epsilon},\, 1\}$, where $r_{\epsilon}$
is the interior critical point. The case $r^{*} = 1$ is
excluded for $\epsilon > \lambda_{-}$ by an analogous endpoint
analysis to that of (R1) above (the trajectory of
$f(\lambda_{\mathrm{cut}}) - f(\epsilon)$ as $r$ approaches $1$
is now strictly positive), so $r^{*} \in \{0, r_{\epsilon}\}$.

We compare $\mathrm{GE}_{\mathrm{int}}(\epsilon) :=
\mathrm{GE}(r_{\epsilon},\epsilon)$ and $\mathrm{GE}(0)$ as
functions of $\epsilon$. Two observations simplify the
computation. First, for fixed $r$, the data-bias term
(\ref{eq:DB-r}) depends on $\epsilon$ only through the
coupling $r = \int_{\lambda_{\mathrm{cut}}}^{\lambda_{+}}
p_{\mathrm{MP}}(\lambda)\,d\lambda$, so at fixed $r$ the cutoff
$\lambda_{\mathrm{cut}}$ and the integrals
$I_{-1}(\lambda_{\mathrm{cut}})$,
$I_{\log}(\lambda_{\mathrm{cut}})$ are all
$\epsilon$-independent, giving
$\partial(\text{Data bias})/\partial\epsilon\big|_{r} = 0$.
Second, by the first-order optimality condition
$\partial\mathrm{GE}/\partial r\big|_{r=r_{\epsilon}} = 0$, the
envelope theorem eliminates the $dr_{\epsilon}/d\epsilon$
contribution. Combining,
\begin{equation*}
  \frac{d\,\mathrm{GE}_{\mathrm{int}}(\epsilon)}{d\epsilon}
  \;=\;
  \frac{\partial \mathrm{ME}}{\partial \epsilon}\bigg|_{r = r_{\epsilon}}
  \;=\;
  \frac{\NV}{2}\,(1 - r_{\epsilon})\,f'(\epsilon),
\end{equation*}
while $\frac{d}{d\epsilon}\mathrm{GE}(0) = \frac{\NV}{2}\,f'(\epsilon)$.
Subtracting,
\begin{equation*}
  \frac{d}{d\epsilon}\big[\mathrm{GE}_{\mathrm{int}}(\epsilon) - \mathrm{GE}(0)\big]
  \;=\;
  -\,\frac{\NV}{2}\,r_{\epsilon}\,f'(\epsilon)
  \;>\; 0
\end{equation*}
on $(\lambda_{-}, 1)$, since $r_{\epsilon} > 0$ and
$f'(\epsilon) = (\epsilon - 1)/\epsilon^{2} < 0$. (Note that
both $\mathrm{GE}_{\mathrm{int}}$ and $\mathrm{GE}(0)$ are
themselves \emph{decreasing} in $\epsilon$; their difference
increases because $\mathrm{GE}(0)$ decreases faster.)
The difference is therefore strictly increasing in $\epsilon$
on $(\lambda_{-}, 1)$, is negative at $\epsilon = \lambda_{-}$
(where $r_{\epsilon} = 1$ and $\mathrm{GE}_{\mathrm{int}}(\lambda_{-}) =
\mathrm{GE}(r{=}1) < \mathrm{GE}(0)|_{\epsilon = \lambda_{-}}$
by the same upper-endpoint calculation as in (R1)), and is
positive as $\epsilon \to 1^{-}$. Hence it changes sign
exactly once at the unique $\epsilon_{*}(\alpha)$ defined by
(\ref{eq:eps-star}). For $\epsilon < \epsilon_{*}$ the interior
wins, giving (R2); for $\epsilon \geq \epsilon_{*}$ the
boundary wins, giving (R3).
\qed
\end{proof}

\begin{remark}
The collapse threshold $\epsilon_{*}(\alpha)$ has a simple but
counter-intuitive operational meaning. A naive expectation
would be that, since $\epsilon < 1$ in the regime of interest
and the pure-noise model $\mathcal{N}(0, \epsilon I_{\NV})$ is
already a strictly misspecified approximation of the truth,
learning \emph{any} positive number of empirical eigendirections
from the data should improve over not learning at all---if only
slightly. Proposition~\ref{prop:phase} shows that this
expectation is wrong above $\epsilon_{*}(\alpha)$: the
finite-sample over-confidence cost of fitting any positive rank
\emph{exceeds} the model-error reduction it provides, so that
the optimal trained model has $\NK^{*} = 0$ and uses \emph{no}
information from the training data. The boundary
$\epsilon_{*}(\alpha)$ is therefore the largest noise floor for
which the data still contain enough information to be worth
using at all. This is the $\epsilon$-PCA analogue of the
\emph{singular learning} regime of \cite{Watanabe2010}, in
which finite-sample fluctuations overwhelm the signal that the
model can extract.
\end{remark}

The phase structure of Proposition~\ref{prop:phase} is plotted
in Figure~\ref{fig:phase-diagram}, with the analytical
boundaries $\lambda_{-}(\alpha)$ and $\epsilon_{*}(\alpha)$
overlaid on a direct numerical determination of the optimal
regime by minimisation of the asymptotic GE on a grid of
$(\alpha, \epsilon)$ values. The numerical regimes (colored
markers) lie inside the analytical regions to within one grid
spacing, providing direct verification of
Proposition~\ref{prop:phase}. The interior phase covers the
bulk of the $(\alpha, \epsilon)$ rectangle and the closed-form
$\lambda_{\mathrm{cut}}^{*} = \epsilon$ rule of
Theorem~\ref{thm:optimal-cutoff} applies throughout it; the
retain-all phase shrinks rapidly with $\alpha$ (because
$\lambda_{-}(\alpha) \to 0$ as $\alpha \to 1$), and the
collapse threshold $\epsilon_{*}(\alpha)$ decreases monotonically
with $\alpha$ from $\sim 0.8$ at $\alpha = 0.1$ to $\sim 0.6$
at $\alpha = 0.9$.

\begin{figure}[!t]
  \centering
  \includegraphics[width=\linewidth]{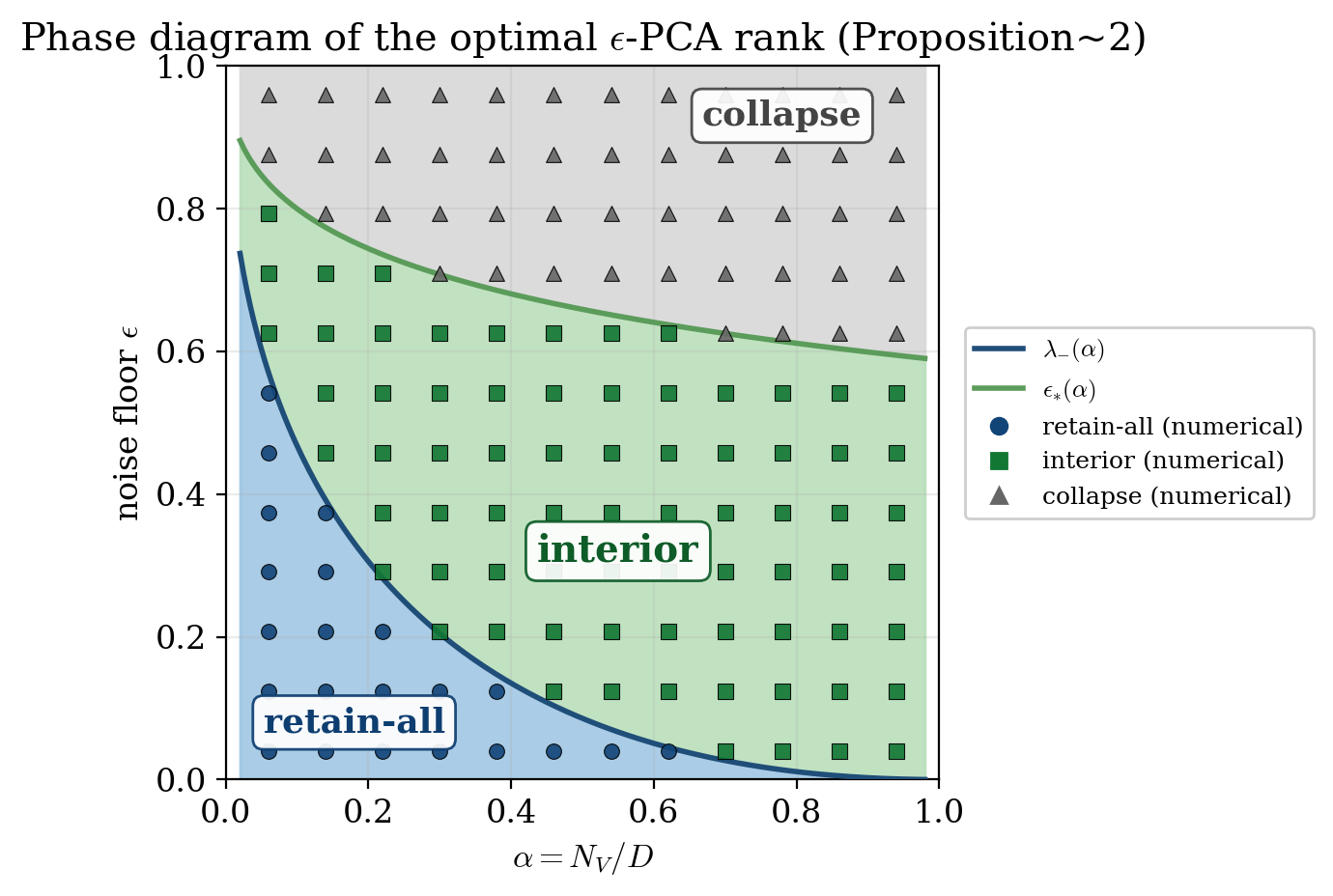}
  \caption{Three-regime phase diagram of the optimal
  $\epsilon$-PCA rank in the $(\alpha, \epsilon)$ plane
  (Proposition~\ref{prop:phase}). The lower boundary
  $\lambda_{-}(\alpha) = (1-\sqrt{\alpha})^{2}$ (dark blue
  curve) separates the retain-all phase ($\NK^{*} = \NV$, blue
  region) from the interior phase
  (Theorem~\ref{thm:optimal-cutoff}, green region); the upper
  boundary $\epsilon_{*}(\alpha)$ (green curve, computed
  numerically as the unique solution of (\ref{eq:eps-star}))
  separates the interior phase from the collapse phase
  ($\NK^{*} = 0$, gray region). Markers show the regime
  obtained by direct minimisation of the asymptotic GE
  (\ref{eq:GE-decomp-asymp}) on a $22\times 22$ grid: blue
  circles for retain-all, green squares for interior, gray
  triangles for collapse. The numerical regimes match the
  analytical boundaries to within one grid spacing.}
  \label{fig:phase-diagram}
\end{figure}

\subsection{Numerical illustration}
\label{sec:epca-numerics}

Figure~\ref{fig:NKstar-eps} verifies the closed-form prediction
of Theorem~\ref{thm:optimal-cutoff} against direct Wishart
sampling. For an isotropic Gaussian source
$\PV = \mathcal{N}(0, I_{\NV})$ at $\NV = 64$, $D = 96$
($\alpha = \NV / D = 2/3$), and $\epsilon = 0.5$, we form $800$
independent empirical covariance matrices $\Sigma^{m}$, build the
\emph{textbook} eigen-$\epsilon$-PCA model
$\QVm = U^{m}D^{m}(U^{m})^{\top}$ of (\ref{eq:QVm-block}) for
each candidate rank $\NK \in \{0,1,\ldots,\NV\}$ using the
empirical eigenvectors $U^{m}$, and average
$\KL(\PV \,\|\, \QVm)$ over realisations. The KL is evaluated
in the full $\NV \times \NV$ matrix form, so that the rotation
$U^{m}$ is used explicitly; the same realisations are also used
to compute the $\diamond$-model GE
$\langle\KL(\PV\|\widetilde{Q}^{m})\rangle_{m}$ in the standard
basis, and the two agree to machine precision
($< 10^{-14}$ in nats over the entire grid). This provides a
direct numerical verification of
Lemma~\ref{lem:rot-equiv}.

The empirical generalization error (black curve in
Fig.~\ref{fig:NKstar-eps}) is a clear U-shape, and the vertical
red dotted line marks the closed-form optimal rank $\NK^{*}$
predicted by Theorem~\ref{thm:optimal-cutoff} via the rule
``retain exactly those empirical eigenvalues exceeding the noise
floor $\epsilon$''. The empirical minimum of the GE curve
coincides with this prediction to within one grid spacing,
verifying the analytic result directly.

The same figure also reports the three components of the
$\diamond$-model decomposition: $\mathrm{ME}(r)$ (red squares,
linear in $r$, computed from the closed form (\ref{eq:ME-r}));
the data bias $\mathrm{Data\ bias}(r)$ (green triangles,
computed from the closed form (\ref{eq:DB-r})); and the
variance (blue triangles, asymptotically vanishing for the
$\diamond$-model). Their sum (gray dashed) overlays the
empirical GE curve to within $4\!\times\!10^{-14}$ in nats over
the entire range $\NK \in [0, \NV]$, providing a direct
numerical verification of Theorem~\ref{thm:3decomp} for the
$\diamond$-model. The crossing of $\mathrm{ME}(r)$ and
$\mathrm{Data\ bias}(r)$ near $\NK = \NK^{*}$ visualises the
marginal-rate balance discussed after Theorem~\ref{thm:optimal-cutoff}:
the closed-form optimum lies where the slopes of the two
components are equal in magnitude.

\begin{figure}[!t]
  \centering
  \includegraphics[width=\linewidth]{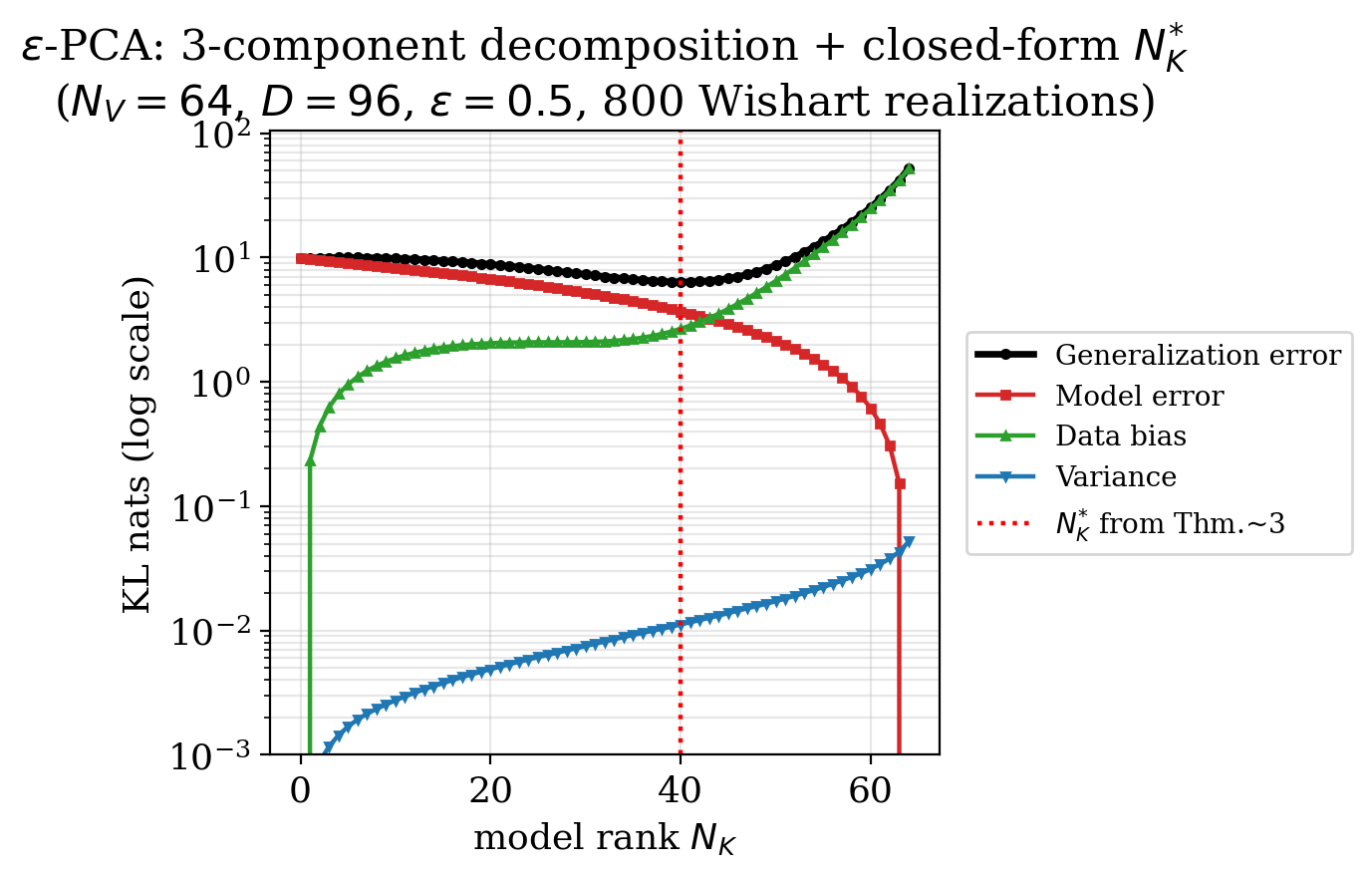}
  \caption{Joint numerical verification of
  Lemma~\ref{lem:rot-equiv},
  Theorem~\ref{thm:3decomp}, and
  Theorem~\ref{thm:optimal-cutoff}, on isotropic Gaussian data
  $\PV = \mathcal{N}(0, I_{\NV})$ at $\NV = 64$, $D = 96$
  ($\alpha = 2/3$), $\epsilon = 0.5$, sample-averaged over $800$
  Wishart realizations.
  \emph{Black:} empirical total generalization error of the
  textbook eigen-$\epsilon$-PCA model $\QVm = U^{m}D^{m}(U^{m})^{\top}$
  of (\ref{eq:QVm-block}), computed in full $\NV\times\NV$ matrix
  form (the rotation $U^{m}$ is used explicitly).
  \emph{Red, green, blue:} the three components
  $\mathrm{ME}(r)$, $\mathrm{Data\ bias}(r)$,
  $\mathrm{Variance}(r)$ of the $\diamond$-model decomposition
  (Theorem~\ref{thm:3decomp} applied to $\Mfam_{\diamond}^{\NK}$),
  computed from the closed forms (\ref{eq:ME-r}), (\ref{eq:DB-r}).
  \emph{Gray dashed:} their sum
  $\mathrm{ME} + \mathrm{Data\ bias} + \mathrm{Variance}$. The
  gray-dashed and black curves overlay each other to within
  $\sim\!10^{-14}$ in nats over the entire range
  $\NK \in [0, \NV]$, providing a direct numerical verification of
  both the additive identity of Theorem~\ref{thm:3decomp} (the
  closed-form $\diamond$-model decomposition reproduces the
  empirical GE) and of Lemma~\ref{lem:rot-equiv} (the textbook
  eigen-$\epsilon$-PCA GE coincides with the $\diamond$-model GE
  on isotropic data).
  \emph{Red dotted vertical line:} the closed-form optimal rank
  $\NK^{*}$ from Theorem~\ref{thm:optimal-cutoff} via the cutoff
  $\lambda_{\mathrm{cut}}^{*} = \epsilon$; it coincides with the
  empirical grid argmin of the GE.}
  \label{fig:NKstar-eps}
\end{figure}

\section{Information-theoretic interpretation}
\label{sec:it-interpretation}

The decomposition (\ref{eq:3decomp-thm}) of the unsupervised
generalization error has natural readings in the language of
classical information theory. We collect two of them here for
completeness; neither is required for the proofs above, but
they make explicit how the decomposition fits within the
standard tools of the Shannon-theoretic literature.

\subsection{Cross-entropy form and code-length deficits}

The KL divergence $\KL(\PV \,\|\, \QVm)$ is, by the Gibbs identity,
the difference between the cross-entropy of $\QVm$ relative to
$\PV$ and the entropy of $\PV$:
\begin{equation}
  \KL(\PV \,\|\, \QVm) \;=\; H(\PV; \QVm) - H(\PV),
  \label{eq:gibbs}
\end{equation}
where $H(\PV; \QVm) := -\int \PV \log \QVm$.
Operationally, $H(\PV; \QVm)$ is the asymptotic per-symbol
description length, in nats, achieved by encoding samples from
$\PV$ with a code matched to $\QVm$
\cite[Chap.~5]{CoverThomas2006}, while $H(\PV)$ is the optimal
code length achievable for $\PV$. The generalization error is
therefore the average code-length \emph{deficit} introduced by
using a finite-sample model in place of the true distribution.
Substituting (\ref{eq:gibbs}) into (\ref{eq:3decomp-thm}) and
isolating the $\PV$-dependent terms gives
\begin{equation}
  \big\langle H(\PV; \QVm)\big\rangle_{m}
  \;=\;
  H(\PV; \QVz)
  \;+\;
  \mathrm{Data\ bias}
  \;+\;
  \mathrm{Variance},
  \label{eq:cross-entropy-decomp}
\end{equation}
where the data bias on the right is the integral
$\int \PV \log(\QVz/\QVbar)$ of (\ref{eq:databias-alg}) and the
variance is $\langle \KL(\QVbar\|\QVm)\rangle_{m}$. In words, the
expected coding cost of the trained model splits into a model-class
floor $H(\PV; \QVz)$, a finite-sample bias of the m-projection,
and a fluctuation term that vanishes when training is deterministic.

\subsection{Rate--distortion view of $\epsilon$-PCA}

The $\epsilon$-PCA model with rank $\NK$ can be read as a
rate--distortion code for the latent Gaussian source $\PV =
\mathcal{N}(0,I_{\NV})$: it spends one ``natural parameter'' per
retained eigenvalue (rate $\propto \NK$) and pays a per-discarded-
direction distortion equal to $\KL\!\big(\mathcal{N}(0,1)\|
\mathcal{N}(0,\epsilon)\big) = \tfrac{1}{2}(f(\epsilon) - 1)$.
The model error (\ref{eq:ME-r}) is then exactly the total
distortion at rate $\NK$, namely
$\mathrm{ME} = \tfrac{1}{2}(\NV - \NK)(f(\epsilon) - 1)$.
The closed-form Theorem~\ref{thm:optimal-cutoff} states that,
when the finite-sample data-bias term (\ref{eq:DB-r}) is also
taken into account, the optimal rate $\NK^{*}$ is the one for
which the marginal cost of admitting one more empirical
direction into the data bias balances the marginal benefit of
removing one $\epsilon$-pinned direction from the model error,
yielding the cutoff condition
$f(\lambda_{\mathrm{cut}}^{*}) = f(\epsilon)$. This is the natural
analogue, in our generative-modelling setting, of the optimal
hard-thresholding rule of \cite{GavishDonoho2014} for low-rank
denoising; here the threshold is set by the model's intrinsic
noise floor $\epsilon$ rather than by the upper edge of the
Marchenko--Pastur support.

\section{Discussion}
\label{sec:discussion}

We have established a closed-form characterization of the optimal
model rank for $\epsilon$-PCA on isotropic Gaussian data
(Theorem~\ref{thm:optimal-cutoff}): the unique interior local
minimum of the generalization error sits at the cutoff
$\lambda_{\mathrm{cut}}^{*} = \epsilon$, so that the optimal model
retains precisely those empirical covariance eigenvalues that
exceed the intrinsic noise floor. The first-order condition
$f(\lambda_{\mathrm{cut}}^{*}) = f(\epsilon)$ with
$f(x) = 1/x + \log x$ also admits a second root
$\widetilde{\lambda} > 1$, but this root is a local
\emph{maximum} of GE on the kept branch rather than a minimum.
A comparison of the interior local-minimum value with the $r = 0$
boundary value of GE yields a sharp three-regime phase structure
for the global optimum (Proposition~\ref{prop:phase}), with the
lower Marchenko--Pastur edge $\lambda_{-}(\alpha)$ separating the
retain-all phase from the interior phase, and an analytically
computable collapse threshold $\epsilon_{*}(\alpha)$ separating
the interior phase from the collapse phase. The closed-form
$\lambda_{\mathrm{cut}}^{*} = \epsilon$ rule and the corresponding
phase diagram are the natural KL-divergence analogues, in a
generative setting, of the spectral truncation rules of
\cite{GavishDonoho2014,Veraart2016MPPCA}.

The framework underlying these closed-form results is an
information-geometric three-component decomposition of the
unsupervised GE, which formalises and extends the empirical
two-component tradeoff of \cite{Kim2023JSTAT}: the data error
splits, via the generalized Pythagorean theorem combined with a
dual variance identity for the e-mixture, into a finite-sample
data bias and a training-stochasticity variance, with each of
the three components non-negative whenever the model manifold
is e-flat (Theorem~\ref{thm:3decomp}). When the model manifold
fails to be e-flat---for instance, for visible marginals of
hidden-variable models or for rank-constrained Gaussian models
such as the eigen-$\epsilon$-PCA class
(\ref{eq:epca-class})---the data bias is no longer forced to be
non-negative, and the framework cannot be invoked directly. We
have shown how to circumvent this for $\epsilon$-PCA itself by a
technical reformulation (Lemma~\ref{lem:rot-equiv}) that places
the trained model on an e-flat sub-family with the same total
GE on isotropic data; the closed-form
Theorem~\ref{thm:optimal-cutoff} is then recovered as a direct
consequence of Theorem~\ref{thm:3decomp}, with the cutoff
condition emerging as a marginal-rate balance between the model
error and the data bias.

More broadly, the framework developed here suggests several
lines of further inquiry that we leave open.

\paragraph{Spike model and anisotropic data.}
The closed-form Theorem~\ref{thm:optimal-cutoff} relies on
isotropy of $\PV$ in two specific places: in
Lemma~\ref{lem:rot-equiv}, which uses the conjugation
invariance of $\KL(I \,\|\, \cdot)$ to transfer the
$\diamond$-model decomposition back to eigen-$\epsilon$-PCA;
and in the m-projection $\widetilde{Q}^{0}$ of $\PV$ onto
$\Mfam_{\diamond}^{\NK}$, which is diagonal with entries $1$ in
the kept block. For a spiked covariance
$\Sigma_{0} = I + \sum_{j=1}^{r} \theta_{j}\,u_{j}u_{j}^{\top}$,
neither step survives literally: the conjugation invariance is
broken by the spike directions, and the m-projection inherits
the spike eigenvalues in the kept block. We expect that the
correct generalisation has a cutoff condition of the form
$f(\lambda_{\mathrm{cut}}^{*}) = f(\sigma_{j}^{2})$, where
$\sigma_{j}^{2}$ is the projected variance of $\PV$ along the
$j$-th surviving spike direction, modified by the BBP transition
\cite{BaikBenArousPeche2005}: spikes whose strength
$\theta_{j}$ exceeds the BBP threshold $\sqrt{\alpha}$ remain
detectable in the empirical spectrum and shift the cutoff
upward, while sub-threshold spikes are absorbed into the bulk
and recover the isotropic answer
$\lambda_{\mathrm{cut}}^{*} = \epsilon$. Establishing this
precisely requires the joint spectral law of spiked Wishart
matrices \cite{BaikBenArousPeche2005,Johnstone2001}, which we
leave for future work. A fully anisotropic generalisation
(arbitrary $\Sigma_{0}$ with no spike structure) is more
delicate and would naturally invoke the free-probability
machinery for the limiting spectrum of
$\Sigma_{0}^{1/2} W \Sigma_{0}^{1/2}$ with $W$ Wishart.

\paragraph{Sharpness of the obstruction.}
Proposition~\ref{prop:negative} shows that the data bias of
Theorem~\ref{thm:3decomp} can become negative whenever the
model manifold fails to be e-flat, but does not say
\emph{when}. Whether negativity occurs for every non-affine
constraint or only for constraints of a specific
type---rank-like, equality-like, or low-dimensional
sub-manifold---would turn the data bias into a quantitative
geometric diagnostic in the sense of
\cite{Watanabe2010,Grunwald2007MDL}, with concrete tests for
the e-flat regime of complex generative models.

\paragraph{Application to fully visible Boltzmann machines.}
The most directly testable extension of the framework is to the
canonical \emph{strict} e-flat fully visible example: the fully
visible Boltzmann machine, whose model class is itself an
exponential family in its natural parameters and therefore
satisfies Theorem~\ref{thm:3decomp} without any technical
reformulation of the kind used in Section~\ref{sec:variant4}.
Such an implementation has long been impractical because the
partition function is intractable and classical Markov chain
Monte Carlo sampling converges slowly; recent work on
\emph{diabatic quantum annealing} for Boltzmann sampling
\cite{Gyhm2024DQABoltzmann} provides a hardware-accelerated
alternative that opens the door to a direct numerical
verification of the U-shaped GE decomposition in this
discrete-state setting, and---through extensions to
latent-variable priors---also to the case in which
Proposition~\ref{prop:negative}'s obstruction is expected to
manifest.

\section*{Acknowledgment}
%

\bibliographystyle{plain}
\bibliography{refs}

\end{document}